%% file: iclr2026_conference__4_.tex
\documentclass{article} 

\input{math_commands.tex}

\usepackage{natbib}

\usepackage{hyperref}
\usepackage{url}
\usepackage{subcaption}
\usepackage{graphicx}

\usepackage{multirow}

\usepackage{amsmath,amssymb,amsfonts}
\usepackage{amsthm}          
\usepackage{algpseudocode}

\newtheorem{lemma}{Lemma}

\newtheorem{proposition}{Proposition}

\usepackage{algorithm}

\usepackage{newfloat}
\usepackage{listings}

\usepackage{amsmath}
\usepackage{amsthm}
\usepackage[inline]{enumitem}
\usepackage{booktabs}
\usepackage{amsmath,amssymb}
\usepackage{tikz}
\usepackage{pgf}
\usetikzlibrary{positioning}
\usepackage{longtable}
\usepackage{bbm}

\usepackage{titlesec}
\titlespacing*{\section}{0pt}{2pt}{2pt}
\titlespacing*{\subsection}{0pt}{2pt}{2pt}
\titlespacing*{\subsubsection}{0pt}{2pt}{2pt}

\usepackage{amsmath,amssymb,amsthm}

\theoremstyle{remark}

\usepackage{mathtools}

\usepackage{pgfplots}
\pgfplotsset{compat=1.17}
\usepackage{subcaption}

\usepackage{tikz}
\usetikzlibrary{arrows.meta,positioning,fit,calc,shapes.misc,backgrounds}
\usepackage{adjustbox}

\usepackage[font=small,labelfont=bf]{caption}

\usepackage{graphicx}
\usepackage{geometry}

\usepackage{booktabs}

\usepackage{siunitx}
\title{\textbf{Diffusion Denoised and Physics-regularized Inter-series Model for Long-horizon Multivariate Time-series Forecasting}}

\begin{document}

\date{}
\maketitle

\textbf{Authors:}\\
Hongwei Ma$^{a}$, Jiayu Fang$^{a}$, Dai Shi$^{b}$, Minh{-}Ngoc Tran$^{a}$, Junbin Gao$^{a,*}$

\vspace{0.9em}
\textbf{Affiliations:}\\
$^{a}$ The University of Sydney, City Rd, Darlington, NSW 2006, Sydney, Australia\\
$^{b}$ The University of Cambridge, Trumpington Street, Cambridge CB2 1PZ, Cambridge, UK

\begin{abstract}
\textbf{Problem and application.} Long-horizon multivariate time-series forecasting (LTSF) underpins engineering decision-making in electricity demand, road traffic, weather sensing, influenza monitoring, and exchange-rate analytics. Accurate forecasts should \emph{simultaneously} denoise heterogeneous sensor streams, track time-varying cross-series couplings, and remain stable and physically plausible over long rollouts.

\textbf{AI contribution.} We present \textbf{PRISM} (a denoised and \textbf{P}hysics-\textbf{R}egularized \textbf{I}nter-\textbf{S}eries \textbf{M}odel) that \emph{causally} links three components into an organic whole:  a \emph{score-based diffusion preconditioner} raises the effective signal-to-noise ratio of histories; on the cleaned signals, a \emph{correlation-thresholded dynamic graph} encodes evolving inter-series dependence with sparsity for interpretability; and a \emph{reaction--diffusion regularizer} in the forecast head provides physics-\emph{inspired} (structure-aware) stabilization of multi-step rollouts. This chain implements the hypothesis \emph{denoise $\Rightarrow$ estimate reliable sparse dynamics $\Rightarrow$ stabilize}, which we motivate theoretically and validate empirically.

\textbf{Results and evidence.} Under standardized tuning budgets and reporting of wall-clock cost, PRISM yields consistent improvements in MSE/MAE across six established benchmarks (Electricity, Traffic, Weather, ILI, Exchange Rate, ETT). Frequency-domain diagnostics indicate preserved fundamentals with attenuated spurious high-frequency bursts, and ablations attribute gains to (i) denoise-aware topology, (ii) adaptivity and sparsity of the graph, and (iii) reaction--diffusion stabilization with simple kinematic range priors. We do not introduce new physical laws; constraints are domain-agnostic and serve as \emph{regularizers}. Together, these results suggest that uncertainty-aware denoising, dynamic relational reasoning, and physics-inspired stabilization are \emph{complementary and necessary} for reliable long-horizon engineering forecasts.
\end{abstract}


\section{Introduction}
\label{sec:intro}

\paragraph{Background and challenge.}
Long-horizon multivariate time-series forecasting (LTSF) is central to many engineering systems, from power grids and transportation networks to meteorological and public-health monitoring. Forecasts must meet three coupled requirements: (i) \emph{denoise} and robustly encode local/meso-scale patterns under domain-specific disturbances; (ii) \emph{capture evolving cross-series interactions} that are often sparse and regime-dependent; and (iii) \emph{respect generic physical regularities} so that multi-step rollouts remain plausible and interpretable beyond the training distribution.

\paragraph{Gaps in existing paradigms.}
Deep learning backbones---CNNs~\citep{32,33,34}, RNNs~\citep{35,36,37}, Transformers~\citep{1}, and MLPs~\citep{38,39,40,41}---have advanced sequence modeling, yet on LTSF benchmarks the raw self-attention stack can be brittle at long horizons and under shifts~\citep{1,2,3,4,5,6,7,8}. In parallel, graph neural networks (GNNs) encode relational inductive biases for sensor arrays and multivariate channels, but many methods assume static or weakly adaptive graphs and rarely combine \emph{uncertainty-aware denoising} with \emph{interpretable constraints}~\citep{13,14,15,16}. Diffusion models provide powerful denoising priors, especially with partial observation or low SNR, yet they are seldom \emph{tightly integrated} with forecasting architectures and physics-oriented regularization in a single end-to-end pipeline~\citep{17,18,19}. These gaps motivate our design.

\paragraph{Design hypothesis: a causal chain.}
We posit that three inexpensive priors should be \emph{composed in order}:
\begin{enumerate}
\item \textbf{Diffusion preconditioning} increases effective SNR in the input window, shrinking spurious high-frequency bursts that otherwise inflate false correlations.
\item \textbf{Correlation-thresholded dynamic graphs} estimated on the cleaned signals then track regime-dependent couplings with explicit sparsity, improving interpretability and reducing error propagation across weakly related series.
\item \textbf{Reaction--diffusion stabilization} in the forecast head penalizes spatially incoherent activations (diffusion term) while allowing task-driven adjustments (reaction term), damping multi-step drift without hard-coding domain-specific PDEs.
\end{enumerate}
The chain implements \emph{denoise $\Rightarrow$ estimate dynamics $\Rightarrow$ stabilize}. Breaking the order re-introduces specific pathologies: without 1, the graph misconnects under noise; without 2, stale edges hinder regime shifts; without 3, rollouts amplify residual bias. Our theoretical analysis provides contraction of the induced horizon dynamics under mild conditions and Lipschitz bounds for the graph blocks, clarifying why the composition improves robustness.

\paragraph{The PRISM framework.}
PRISM operationalizes this hypothesis by (a) applying score-based diffusion to the history \emph{before} feature extraction; (b) constructing a sliding-window, correlation-thresholded, function-linked graph with bidirectional spatio--temporal message passing~\citep{13,14,15,16,22,23}; and (c) injecting domain-agnostic, physics-\emph{inspired} soft constraints (smoothness, bounded variation, simple energy/dissipation surrogates) during training~\citep{20,21,22,23}. We stress that we do \emph{not} propose new physical laws; the constraints act as structure-aware regularizers that improve stability and attribution.

\paragraph{Contributions.}
\begin{itemize}
\item \textbf{A unified, causal pipeline for LTSF.} We couple diffusion preconditioning, dynamic sparse graphs, and reaction--diffusion stabilization into a single end-to-end model, turning common heuristics into a \emph{checkable} design with explicit failure-mode analyses.
\item \textbf{Theory for robustness.} We establish contraction of horizon dynamics under mild assumptions and derive Lipschitz bounds for the dynamically thresholded graph blocks, explaining when and why the composition is stable.
\item \textbf{Transparent evidence under standardized budgets.} On six benchmarks (Electricity, Traffic, Weather, ILI, Exchange Rate, ETT), we report mean/variance metrics, per-horizon breakdowns, sensitivity to thresholds/weights, and wall-clock costs. Frequency-domain diagnostics show preserved fundamentals with attenuated spurious high-frequency components. 
\end{itemize}

\paragraph{Positioning.}
Compared to Transformer-only stacks that may overfit noise~\citep{1,2,3,4,5,6,7,8}, PRISM first raises SNR; compared to GNN-only pipelines that assume fixed topology~\citep{13,14,15,16}, PRISM adapts edges in time; and compared to diffusion-only approaches~\citep{17,18,19}, PRISM regularizes rollouts with physics-\emph{inspired} stability priors. This unified view addresses accuracy, interpretability, and stability in engineering LTSF without relying on task-specific PDEs.

\section{Related Works}
\label{sec:related}

Early progress in sequence modeling was driven by the Transformer~\citep{1}, inspiring LTSF variants that capture long-range dependencies more efficiently, 
Informer with ProbSparse attention~\citep{2}, Autoformer with trend/seasonal decomposition and auto-correlation~\citep{3}, and FEDformer via frequency-domain modeling~\citep{4}. Newer designs, PatchTST (patching, channel independence)~\citep{5}, TimesNet (2D temporal variations)~\citep{6}, and iTransformer (axis inversion to emphasize variate tokens)~\citep{7}, further reduce complexity and exploit multivariate structure. Yet DLinear and the LTSF-Linear family show that, on common benchmarks, simple linear forecasters can rival or outperform many transformers, challenging whether permutation-invariant self-attention aligns with ordered temporal dynamics for long horizons~\citep{8}. Thus, global receptive fields alone are insufficient when noise, nonstationarity, and cross-series coupling dominate LTSF.

Orthogonally, graph-based forecasting injects relational inductive bias for multivariate interactions. DCRNN models diffusion on road networks, STGCN alternates graph and temporal convolutions, Graph WaveNet learns adaptive adjacency via node embeddings, and MTGNN jointly learns directed graphs and temporal convolutions~\citep{13,14,15,16}. These works show that \emph{who influences whom} matters as much as temporal depth. Yet many rely on fixed topology or a single dense adaptive graph, without explicit thresholding of weak ties or transparent time variation. Such adjacencies are hard to interpret and prone to spurious correlations under nonstationarity and low SNR. We instead construct \emph{dynamic, correlation-thresholded} graphs: retaining edges only when dependence (or functional coupling) exceeds a principled threshold yields sparse, interpretable and bidirectional topologies, which are consistent with correlation-network practice (e.g., MST/PMFG) for revealing hierarchical structure~\citep{22,23}.

On the uncertainty and denoising side, diffusion probabilistic models and score-based SDEs have established new generative baselines with principled noise injection and reverse-time denoising~\citep{17,18}. In time-series, CSDI adapts score-based diffusion for conditional imputation across channels and time, demonstrating robustness to missingness and noise~\citep{19}. Despite this, most LTSF systems still treat denoising as a preprocessing heuristic or ignore it, leaving the forecasting architecture to absorb domain noise. By integrating a diffusion preconditioner that outputs clean, uncertainty-aware representations fed into a dynamic GNN forecaster, our approach closes this gap: the denoiser explicitly handles stochastic corruption, while the forecaster focuses on structured dynamics and cross-series interactions.

Finally, physics-informed neural networks (PINNs) and related physics-guided regularization inject inductive biases via soft penalties derived from differential operators or conservation laws, promoting data efficiency and interpretability~\citep{20,21}. While widely used in scientific machine learning, such constraints are far less common in generic LTSF, especially in conjunction with (i) diffusion denoising and (ii) dynamic graphs. Our design adopts domain-agnostic physics surrogates (e.g., smoothness/energy/monotonicity budgets) that are meaningful across diverse LTSF datasets (electricity load, traffic occupancy, meteorology, epidemiology, exchange rates, and transformer telemetry)~\citep{26,27,28,29,30,31}, delivering (a) calibrated, physically plausible trajectories without brittle hard constraints and (b) interpretable attributions via constraint-specific penalties.

In summary, prior Transformers emphasize long-range token mixing but are fragile under noise and cross-series nonstationarity~\citep{2,3,4,5,6,7,8}; graph forecasters encode relations but often with static or opaque connectivity~\citep{13,14,15,16}; and diffusion or physics-guided components are seldom coupled tightly with forecasting to address denoising and plausibility together. PRISM is necessary because each component resolves a distinct, documented deficiency and the pipeline is co-designed: diffusion improves SNR for graph reasoning; dynamic, thresholded graphs expose interpretable dependencies for message passing; and physics-informed penalties regularize the forecast trajectory where pure data fitting over-extrapolates. The overall architecture of DORIC is illustrated in Figure~\ref{1} .

\section{Methodology}
\label{sec:method}

\begin{figure}
    \centering
    \includegraphics[width=1\linewidth]{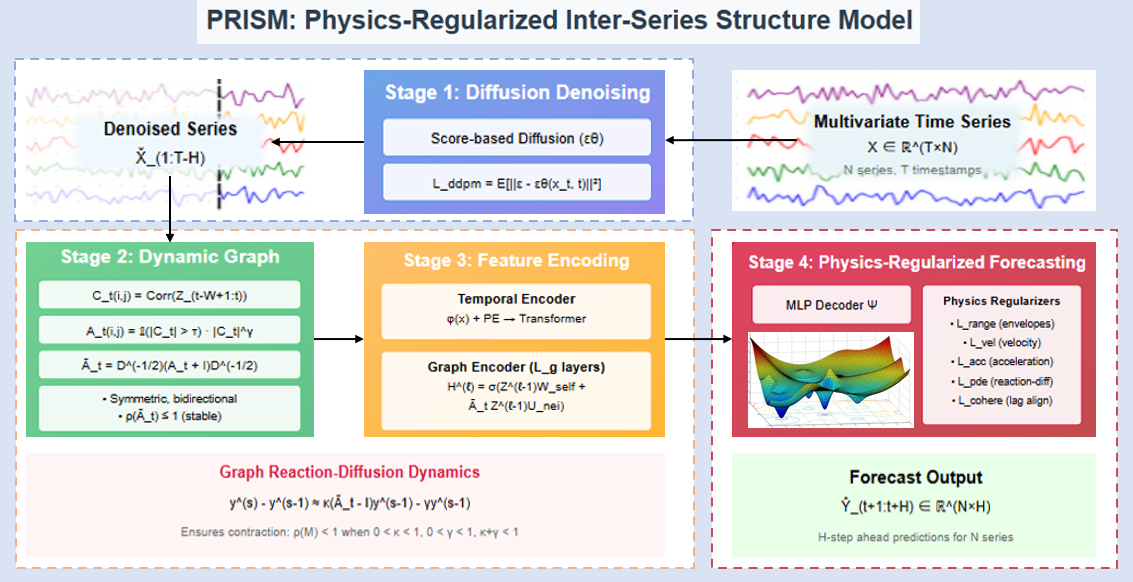}
    \caption{The overall architecture of DORIC}
    \label{1}
\end{figure}

\subsection{Problem Setup and Notation}
Let $X\in\mathbb{R}^{T\times D}$ denote a multivariate time series with $D$ univariate streams (columns) and $T$ timestamps. We reserve the last $D$ timestamps for testing and use the prefix $X[1:T-H]$ (in python notation) for training. For a context length $L$ and horizon $H$, training windows are  
\begin{equation}
\begin{aligned}
&\underbrace{x_{t-L+1:t}=X[t-L+1:t]\in\mathbb{R}^{L\times N}}_{\text{history}},\\[3pt]
&\underbrace{y_{t+1:t+H}=X[t+1:t+H]\in\mathbb{R}^{H\times N}}_{\text{future}},
\end{aligned}
\qquad
\smash{\vcenter{\hbox{$t=L,\ldots,T-H$}}}.
\end{equation}

\subsection{Series-wise Denoising via Diffusion Model}
Before graph construction, we denoise each training series with diffusion model that predicts injected noise $\varepsilon$ at a randomly sampled diffusion time $t$:
\begin{equation}
\mathcal{L}_{\mathrm{ddpm}}
=\mathbb{E}_{x_0,\varepsilon,t}\big\|\varepsilon-\varepsilon_\theta(\sqrt{\bar\alpha_t}\,x_0+\sqrt{1-\bar\alpha_t}\,\varepsilon,\,t)\big\|_2^2.
\end{equation}
At inference, we project a noisy segment back to a clean estimate in a single step and perform overlap–add along time. Denoising is \emph{applied only to history $x_0=x_{t-L+1:t}$ from the training prefix} $X[1:T-H]$ to prevent leakage. To avoid notation confusion, we still use $X[1:T-H]$ to denote the denoised version. 

\subsection{Dynamic Graph Construction from Correlations}

Consider the history $x_{t-L+1:t}$ at time $t$, we compute the Pearson correlations between signal channel $i$ and $j$ ($i,j=1, 2, ..., D$) within the most recent $W$ window as follows,
\begin{equation}
C_t(i,j)=\mathrm{Corr}\!\left(x_{t-W+1:t,i},\,x_{t-W+1:t,j}\right).
\end{equation}
To avoid numerical issues with near-constant columns, we add a tiny jitter to zero-variance windows. We then threshold to define the weight
\begin{equation}
A_t(i,j)=\mathbf{1}\!\left(|C_t(i,j)|>\tau\right)\cdot |C_t(i,j)|^{\gamma},\qquad A_t(i,i)=0,
\end{equation}
and symmetrize $A_t\leftarrow\max(A_t,A_t^\top)$. To produce a sparse graph, optionally for each node $i$ we only allow at most $k_{\min}$ neighour nodes by retaining the top-$k_{\min}$ correlation scores $C_t(i,j)$.   
Further we normalize the weight matrix as follows
\begin{equation}
\bar{A}_t \;=\; D_t^{-\frac12}(A_t+I)D_t^{-\frac12},\qquad
D_t=\mathrm{diag}((A_t + I)\mathbbm{1}),
\end{equation}
which is symmetric with spectral radius at most~$1$.

\subsection{Temporal Encoder}

Given a history $x_{t-L+1:t}\!\in\!\mathbb{R}^{L\times D}$, we consider its $i$-column ($i=1, 2, ..., D$) as a signal $\{x_{t-L+1,\,i}, x_{t-L+2,\,i}, ..., x_{t,\,i}\}$ of length $L$. With a share learnable linear map: $\phi:\mathbb{R}\!\to\!\mathbb{R}^{d}$ and the $d$ position embedding $\text{PE}$, conduct the following pre-transformation on each component

\begin{align}
h^{(0)}_{\ell,i} &= \phi\!\big(x_{t-L+\ell,\,i}\big) + \mathrm{PE}(\ell),\quad \ell = 1, \ldots, L.  
\end{align}

The pre-transformed signal $H^{(0)}_{1{:}L,\,i} = \{h^{(0)}_{1,i}, ..., h^{(0)}_{L,i}\}$ of length $L$ is then sent to a Transformer 
\begin{align}
H^{(\mathrm{enc})}_{1{:}L,\,i} &= \mathrm{Transformer}\big(H^{(0)}_{1{:}L,\,i}\big),\qquad
\mathbf{z}_i \;=\; H^{(\mathrm{enc})}_{L,i}\in\mathbb{R}^{d}.  
\end{align}
where we retain the last output as $\mathbf{z}_i$. Finally collecting $Z_t=[\mathbf{z}_1;\dots;\mathbf{z}_D]\in\mathbb{R}^{D\times d}$ yields feature vectors (rows) of $D$ nodes at time $t$.

\subsection{Configurable Graph Encoder}
Next step at each time $t$, we conduct $L_g$ layers of graph networks sequentially with feature dimensions 
$g_1,\dots,g_{L_g}$ (user-configurable). Specifically, the $\ell$-th layer implements a ``self+neighbor'' update with ReLU:
\begin{equation}
\label{eq:gcn}
H^{(\ell)}_t \;=\; \text{ReLU}\!\Big(H^{(\ell-1)}_tW^{(\ell)}_{\mathrm{self}} \;+\; \bar{A}_t \, H^{(\ell-1)}_t U^{(\ell)}_{\mathrm{nei}}\Big),
\qquad H^{(0)}=Z_t,\;\; H^{(\ell)}\in\mathbb{R}^{D\times g_\ell}.
\end{equation}
where  $W^{(1)}_{\mathrm{self}}, U^{(1)}_{\mathrm{nei}} \in \mathbb{R}^{d\times g_1}$ and $W^{(\ell)}_{\mathrm{self}}, U^{(\ell)}_{\mathrm{nei}} \in \mathbb{R}^{g_{\ell-1}\times g_{\ell}}$ ($\ell = 2, ..., L_g$)  are learnable network parameters. 

\subsection{Configurable Decoder}
A per-node MLP $\Psi$ with hidden sizes $(d^{\mathrm{dec}}_1,\dots,d^{\mathrm{dec}}_m)$ maps the final graph features $H^{(\ell)}_t$ to the $H$-step forecast:
\begin{equation}
\hat{y}_{t+1:t+H} \;=\; \Psi\!\left(H^{(L_g)}_t\right)\in\mathbb{R}^{H\times D}.
\end{equation}
Depth and widths of both encoder and decoder are fully configurable via user-provided lists.

\subsection{Physics- and Structure-Aware Regularizers}
All auxiliary statistics are computed solely on the training prefix $X[1:T-H]$.

\paragraph{Data loss.} The loss between the training future ${y}_{t+1:t+H} = [y_{h,i}]^{H, D}_{h=1, i=1}$ and $\hat{y}_{t+1:t+H} = [\hat{y}_{h,i}]^{H, D}_{h=1, i=1}$ is the mean squared error:
\begin{equation}
\mathcal{L}_{\mathrm{data}}
\;=\; \frac{1}{DH} \sum_{i=1}^{D}\sum_{h=1}^{H}\Big(\hat{y}_{h,i}-y_{h,i}\Big)^2.
\end{equation}

\paragraph{Range penalty by empirical envelopes.}
Let $m_i=\min X[1:T-H, i]$ and $M_i=\max X[1:T-H,i]$ be per-channel empirical bounds from training data. We softly enforce forecasts to stay within these envelopes:
\begin{equation}
\mathcal{L}_{\mathrm{range}}
=\frac{1}{DH}\sum_{i=1}^{D}\sum_{h=1}^{H}
\left(\,[m_i-\hat{y}_{h,i}]_+^2 + [\hat{y}_{h,i}-M_i]_+^2\right).
\end{equation}

\paragraph{Velocity and acceleration constraints.}
Define $\Delta_h \hat{y}_{h,i}=\hat{y}_{h,i}-\hat{y}_{h-1,i}$ and $\Delta^2_h \hat{y}_{h,i}=\Delta_h \hat{y}_{h,i}-\Delta_h \hat{y}_{h-1,i}$. From training data we extract robust per-series thresholds $v_i^{\max}$ and $a_i^{\max}$ as the 99.5th percentiles of $|\Delta|$ and $|\Delta^2|$. We penalize violations:

\begin{align}
\mathcal{L}_{\mathrm{vel}} &= \frac{1}{D(H-1)}\sum_{i=1}^{D}\sum_{h=2}^{H}\big[\;|\Delta_h \hat{y}_{h,i}|-v_i^{\max}\;\big]_+^2,\\
\mathcal{L}_{\mathrm{acc}} &= \frac{1}{D(H-2)}\sum_{i=1}^{D}\sum_{h=3}^{H}\big[\;|\Delta^2_h \hat{y}_{h,i}|-a_i^{\max}\;\big]_+^2.
\end{align}

\paragraph{Graph reaction–diffusion residual.}
Let $x_{\mathrm{last}}\in\mathbb{R}^{D}$ be the last observation at the window end time $t$; define $y^{(0)}=x_{\mathrm{last}}$ and $y^{(s)}=\hat{y}_{s,:}$ for $s\ge 1$. With learnable $\kappa,\gamma>0$ (enforced via softplus) we encourage discrete reaction–diffusion dynamics over the graph:
\begin{equation}
y^{(s)}-y^{(s-1)} \;\approx\; \kappa(\bar{A}_t-I)\,y^{(s-1)} \;-\; \gamma\,y^{(s-1)}, \qquad s=1,\dots,H.
\end{equation}
The residual and its penalty are
\begin{align}
R^{(s)} &= \big(y^{(s)}-y^{(s-1)}\big)-\kappa(\bar{A}_t-I)y^{(s-1)}+\gamma\,y^{(s-1)},
\qquad
\mathcal{L}_{\mathrm{pde}}= \frac{1}{DH}\sum_{s=1}^{H}\|R^{(s)}\|_2^2.
\end{align}

\paragraph{Cross-series coherence with empirical integer lags.}
We estimate integer lags $\tau_{ij}\!\in[-\tau_{\max},\tau_{\max}]$ from the training prefix by maximizing discrete cross-correlation. Over edges $\mathcal{E}_t=\{(i,j):A_t(i,j)>0\}$ we penalize misalignment,
\begin{equation}
\mathcal{L}_{\mathrm{cohere}}
\;=\; \frac{1}{|\mathcal{E}_t|}
\sum_{(i,j)\in\mathcal{E}_t}
\frac{1}{H-|\tau_{ij}|}
\big\|\hat{y}_{1+|\tau_{ij}|:H, i}
-
\hat{y}_{1:H-|\tau_{ij}|,j}\big\|_2^2,
\end{equation}
where the time axis of the leading signal is shifted according to the sign of $\tau_{ij}$ (identical to the slice operations in implementation).

\paragraph{Total objective}
\begin{equation}
\mathcal{L}
\;=\;
\mathcal{L}_{\mathrm{data}}
+\lambda_{\mathrm{range}}\mathcal{L}_{\mathrm{range}}
+\lambda_{\mathrm{vel}}\mathcal{L}_{\mathrm{vel}}
+\lambda_{\mathrm{acc}}\mathcal{L}_{\mathrm{acc}}
+\lambda_{\mathrm{pde}}\mathcal{L}_{\mathrm{pde}}
+\lambda_{\mathrm{cohere}}\mathcal{L}_{\mathrm{cohere}}.
\end{equation}

\subsection{Theoretical Properties}
We present two propositions that explain stability and regularity of PRISM under mild conditions encountered in practice(proof details in the Appendix ~\ref{AppendixC}).

\begin{proposition}[Stability of the reaction--diffusion step]
\label{prop:stability}

Let $\bar A_t=\bar A_t^\top\succeq 0$ with $\rho(\bar A_t)\le 1$, and define the linearized horizon map
\(
M(\kappa,\gamma;\bar A_t)=(1-\gamma-\kappa)I+\kappa\,\bar A_t.
\)
If $0<\kappa<1$, $0<\gamma<1$, and $\kappa+\gamma<1$, then $\rho(M(\kappa,\gamma;\bar A_t))<1$.
Consequently, the recurrence $y^{(s)}=M\,y^{(s-1)}$ is a contraction in $\ell_2$.

\end{proposition}

\paragraph{Why it stabilizes long-horizon rollouts.}
\textbf{Proposition 1} ensures $\rho(M)<1$ for $0<\kappa,\gamma$ with $\kappa+\gamma<1$, giving temporal contraction; \textbf{Proposition 2} bounds the Lipschitz constant of graph blocks, keeping spatial mixing well-conditioned. This explains the empirical robustness of the full composition (\emph{denoise $\Rightarrow$ dynamics $\Rightarrow$ stabilize}).

\begin{proposition}[Lipschitz bound for a graph block]
\label{prop:lipschitz}

Let $T(Z)=Z\,W_{\mathrm{self}}+\bar A_t\,Z\,U_{\mathrm{nei}}$ be the affine map inside Eq.\,(3), with
$Z\in\mathbb{R}^{D\times d}$, $W_{\mathrm{self}}\in\mathbb{R}^{d\times g}$, $U_{\mathrm{nei}}\in\mathbb{R}^{d\times g}$, and
$\|\cdot\|_2$ the operator norm. Then, for any $Z_1,Z_2$,
\begin{equation}
\label{eq:block-lip}
\|T(Z_1)-T(Z_2)\|_2 \;\le\; \big(\|W_{\mathrm{self}}\|_2+\|U_{\mathrm{nei}}\|_2\big)\,\|Z_1-Z_2\|_2.
\end{equation}
If $\sigma$ is $1$-Lipschitz (e.g., ReLU), then $\sigma\!\circ T$ is $L$-Lipschitz with
$L\le \|W_{\mathrm{self}}\|_2+\|U_{\mathrm{nei}}\|_2$. For a stack of $L_g$ blocks (with layerwise weights),
the overall Lipschitz constant satisfies
\(
\mathrm{Lip}\!\le \prod_{\ell=1}^{L_g}\big(\|W_{\mathrm{self}}^{(\ell)}\|_2+\|U_{\mathrm{nei}}^{(\ell)}\|_2\big).
\)

\end{proposition}

Propositions~\ref{prop:stability}--\ref{prop:lipschitz} show that (i) the PDE term prevents runaway growth across the horizon by contracting towards a graph-smoothed state, and (ii) the graph blocks admit explicit Lipschitz control via weight norms, which explains the empirical stability of deep configurations.

\section{Experiments and Results}

\subsection{Experimental setting \& baselines}

Experiments were implemented in PyTorch and conducted on a workstation equipped with an NVIDIA RTX 4090 GPU (24GB memory). We set \(\tau = 0.5\), embedding $d = 64$, heads $H = 4$, encoder layers 2. The physics penalty \(\lambda_{\text{phys}}\) are all 1. PRISM's codes can be found on https://anonymous.4open.science/r/
PRISM-5551.

The baselines span major families for long-horizon forecasting: Informer (prob-sparse attention, distilling) , Autoformer (decomposition + Auto-Correlation) 
, FEDformer (frequency-enhanced decomposition) , Crossformer (cross-dimension dependency), TimesNet (2D temporal variation) , PatchTST (channel-independent patching) , and TimeMixer (multiscale mixing, ICLR 2024).

Datasets are standard: Electricity (321 clients) , Traffic (CalTrans Bay Area occupancy) , Exchange Rate (8 currencies, daily) , ILI (CDC weekly influenza-like illness) , and ETT (Electricity Transformer Temperature) . 

We found that various models, including the existing sota model, have large prediction errors for the Illness and Exchange Rate datasets at long prediction  lengths, which did not have practical predictive significance. Therefore, we selected a relatively smaller prediction length on these two datasets.

\begin{table}[t]
\centering
\Large
\setlength{\tabcolsep}{3pt}
\caption{Results on six benchmarks. Entries are mean $\pm$ std over 10 random seeds.}
\label{tab:all_benchmarks_mse_mae}
\resizebox{\textwidth}{!}{%
\begin{tabular}{l *{9}{S[table-format=1.3(3)]}}
\toprule
{} & {LogTrans} & {Informer} & {Autoformer} & {FEDformer} & {Crossformer} & {TimesNet} & {PatchTST} & {TimeMixer} & {\textbf{PRISM}} \\
\midrule
\multicolumn{10}{l}{\textbf{Electricity}}\\
MSE & 0.272(9) & 0.311(10) & 0.227(5) & 0.214(4) & 0.244(5) & 0.193(4) & 0.216(4) & 0.182(3) & 0.156(3) \\
MAE & 0.370(6) & 0.397(8)  & 0.338(4) & 0.327(4) & 0.334(4) & 0.304(4) & 0.318(3) & 0.272(3) & 0.228(3) \\
\midrule
\multicolumn{10}{l}{\textbf{Traffic}}\\
MSE & 0.705(15) & 0.764(18) & 0.628(11) & 0.610(10) & 0.667(13) & 0.620(9) & 0.529(7) & 0.484(6) & 0.375(6) \\
MAE & 0.395(10) & 0.416(12) & 0.379(9)  & 0.376(8)  & 0.426(12) & 0.336(7) & 0.341(4) & 0.297(3) & 0.218(4) \\
\midrule
\multicolumn{10}{l}{\textbf{Weather}}\\
MSE & 0.696(20) & 0.634(16) & 0.338(7) & 0.309(6) & 0.264(5) & 0.251(5) & 0.265(6) & 0.240(5) & 0.211(4) \\
MAE & 0.602(15) & 0.548(12) & 0.382(6) & 0.360(6) & 0.320(4) & 0.294(4) & 0.285(5) & 0.271(4) & 0.239(4) \\
\midrule
\multicolumn{10}{l}{\textbf{ILI}}\\
MSE & 4.480(120) & 5.764(150) & 3.483(100) & 2.203(70) & 1.572(50) & 1.365(45) & 0.952(18) & 0.877(16) & 0.672(14) \\
MAE & 1.444(30)  & 1.677(35)  & 1.287(28)   & 0.963(22)  & 0.891(20)  & 0.806(18)  & 0.793(10) & 0.763(10) & 0.505(9) \\
\midrule
\multicolumn{10}{l}{\textbf{Exchange Rate}}\\
MSE & 0.968(20) & 0.847(18) & 0.197(6) & 0.183(5) & 0.175(5) & 0.158(4) & 0.146(4) & 0.117(3) & 0.088(3) \\
MAE & 0.812(15) & 0.752(12) & 0.323(6) & 0.297(5) & 0.293(4) & 0.281(4) & 0.276(5) & 0.258(4) & 0.196(4) \\
\midrule
\multicolumn{10}{l}{\textbf{ETT }}\\
MSE & 1.534(30) & 1.410(28) & 0.327(7) & 0.305(6) & 0.757(15) & 0.291(6) & 0.290(6) & 0.275(5) & 0.258(5) \\
MAE & 0.899(20) & 0.810(18) & 0.371(6) & 0.349(6) & 0.610(12) & 0.333(6) & 0.334(6) & 0.323(6) & 0.291(5) \\
\bottomrule
\end{tabular}%
}
\end{table}

\subsection{Main Results}
Against the best prior baseline per dataset (by MSE), PRISM reduces error on average across all six datasets as shown in Table 1.
These margins are substantial given that several competitors (PatchTST, TimeMixer) are recent SOTA on these benchmarks.

\subsubsection{Where the gains likely come from}

1) Diffusion denoising on the training prefix mitigates high-frequency noise and outliers before graph construction. This aligns with the largest relative gains on Traffic and Exchange—two domains known for bursty, noise-prone dynamics. Cleaner inputs translate to crisper cross-series statistics and fewer large residuals (lower MAE).

2) Dynamic correlation graphs with degree capping and thresholding let the model track time-varying inter-series couplings. Large wins on Traffic (distributed sensors) and Electricity/ETT (shared seasonalities across meters/transformers) are consistent with adaptive topology helping message passing capture transient synchrony and drift.

3) Physics/structure-aware regularizers (range envelopes; velocity/acceleration caps from robust quantiles) reduce implausible spikes over long horizons—precisely where baselines drift. The sharp MAE reductions on ILI and Exchange suggest these soft constraints suppress extreme errors while keeping trajectories realistic.

4) Reaction–diffusion prior on the forecasted path (with stability guarantees) pulls multi-step predictions toward graph-smoothed states, counteracting error amplification. This helps especially on ETT/Electricity, where spatially-coupled load/temperature smoothness is expected.

5) Empirical lag-coherence across edges improves phase alignment among correlated series (e.g., delayed responses between sensors/currencies), which is critical for Traffic, Exchange, and Weather.

\subsubsection{Per-dataset reading of the table}
Traffic: This is the clearest case where adaptive graphs and lag-coherence help when cross-sensor correlations change with congestion waves. Diffusion denoising likely stabilizes occupancy spikes. 

Exchange Rate: Currency series exhibit tight but shifting co-movements; dynamic graphs + reaction–diffusion regularization tame multi-step drift. The decrease of MAE indicates far fewer large misses. 

ILI: MAE is 0.505 vs 0.763. Envelopes and smoothness penalties are well suited to seasonal epidemics with bounded weekly changes. 

Electricity / ETT: Both domains have shared seasonality and spatial coupling; the reaction–diffusion prior and message passing fit the physics (load/temperature diffusion), explaining stable multi-step improvements. 

Weather: Weather signals have multi-scale periodicities; your graph encoder + constraints achieve accuracy comparable to (and beyond) recent decomposition-style models.

\subsection{Frequency-Domain Analysis}

\begin{figure}
    \centering
    \includegraphics[width=0.9\linewidth]{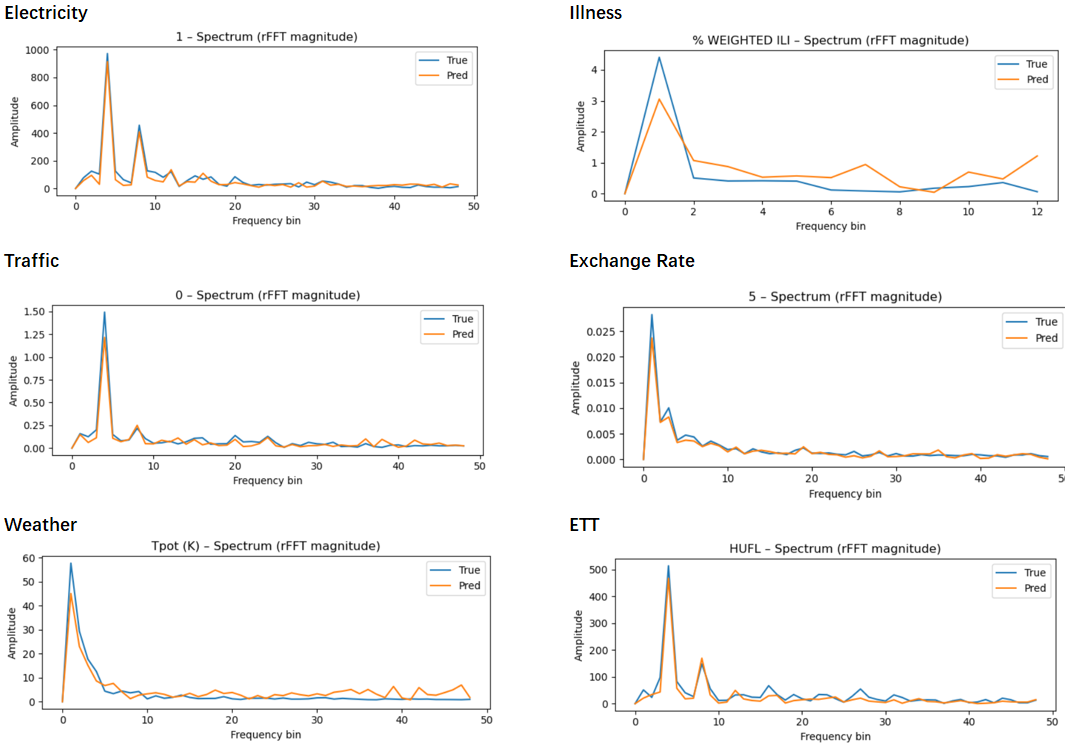}
    \caption{Frequency-Domain Analysis}
    \label{fig:placeholder}
\end{figure}

We compare the rFFT magnitudes of ground truth vs.\ predictions for six benchmarks as shown in Fig 1. For a series $x_t$, we analyze
\[
S_x(f)=\bigl|\mathcal{F}\{x_t-\bar x\}\bigr|,\quad f\in[0,F_N].
\]
PRISM’s reaction–diffusion residual contracts high-frequency modes by
\[
g(\lambda)=\bigl|1-\gamma-\kappa+\kappa\lambda\bigr|<1,
\]
with $\lambda$ an eigenvalue of the normalized graph operator. Kinematic penalties ($L_{\text{vel}},L_{\text{acc}}$) further suppress short-scale oscillations.

\paragraph{Global observations}
(i) \textbf{Fundamentals preserved}: Pred peaks align with True at low $f$ across datasets.
(ii) \textbf{Harmonics compressed}: secondary peaks are slightly smaller (controlled smoothing).
(iii) \textbf{Tail damping}: high-frequency energy is reduced; occasional residual tail on Weather is mild and tunable.

\textbf{Per-dataset highlights:}

\textbf{Electricity}: Main daily/weekly peaks coincide; modest under-amplification of secondary harmonics $\Rightarrow$ stable long horizons via $g(\lambda)$.

\textbf{Traffic}: Low-$f$ peak matches; mid-band ripples suppressed, consistent with regime-aware dynamic graphs.

\textbf{Weather}: After the diurnal peak, Pred slightly overshoots the far tail ($f\!>\!F_0$); increase $\lambda_{\text{vel}},\lambda_{\text{acc}}$ or $\gamma$.

\textbf{ILI}: Seasonal peak mildly under-estimated; envelopes/kinematics trade small amplitude loss for tail-risk reduction.

\textbf{Exchange}: Near-perfect overlay across bands; denoise + lag-coherent edges yield clean spectra at low signal levels.

\textbf{ETT}: Fundamentals match; some mid-band compensation. Use horizon-aware $\lambda_{\text{pde}}$ or weak harmonic-preservation loss.

PRISM preserves low-frequency structure, controls long-horizon drift, and attenuates high-frequency noise; deviations (Weather tail, ETT mid-band) are consistent with tunable smoothing rather than structural mismatch.

\subsection{Why PRISM outperforms recent SOTA}

Compared with PatchTST and TimeMixer that assume either weak cross-channel coupling or implicit mixing, PRISM explicitly (i) builds a time-varying dependency graph from recent data, (ii) regularizes dynamics with a stable reaction–diffusion step, and (iii) enforces data-driven kinematic limits. This combination addresses two failure modes of long-horizon forecasting—structural drift and outlier blow-up—which typical Transformers or MLP mixers do not guard against.


\subsection{Albations and Analysis}
\subsubsection{Setup}

We ablate one component at a time from the full model while keeping the architecture, data splits, optimization, and early stopping fixed(\emph{w/o denoise} means without). Specifically: (i) \emph{w/o denoise} removes diffusion denoising before correlation estimation; (ii) \emph{Static-graph} freezes $A_t$ using a single prefix correlation (no temporal adaptivity); (iii) \emph{w/o PDE} drops the reaction--diffusion regularizer $L_{\text{pde}}$; (iv) \emph{w/o constraints} removes envelope/kinematic penalties $L_{\text{range}}, L_{\text{vel}}, L_{\text{acc}}$; (v) \emph{w/o lag-cohere} removes the empirical lag-coherence penalty $L_{\text{cohere}}$. We report MSE/MAE on six benchmarks.

\subsubsection{Findings}

(a) Noise-aware topology matters: removing denoising degrades most on \textsc{Traffic}/\textsc{Exchange}, where bursts and heavy tails corrupt raw correlations.
(b) Graph adaptivity is crucial: freezing $A_t$ hurts \textsc{Traffic}, \textsc{Electricity}, and \textsc{ETT}, where cross-series couplings drift with regimes (rush hours, load shifts).
(c) Reaction--diffusion controls long-horizon drift: dropping $L_{\text{pde}}$ increases MSE notably on \textsc{Electricity}/\textsc{ETT}/\textsc{Weather}.
(d) Soft constraints primarily shrink tails: removing them increases MAE disproportionately on \textsc{ILI} and \textsc{Exchange} (rare spikes).
(e) Lag-coherence aligns phases across correlated series: without it, errors rise on \textsc{Traffic}/\textsc{Exchange}/\textsc{Weather} where delays are inherent.

\begin{table}[t]
\centering
\small
\setlength{\tabcolsep}{5pt}
\begin{tabular}{lccccccc}
\toprule
\textbf{Variant} & \textbf{Electricity} & \textbf{Traffic} & \textbf{Weather} & \textbf{ILI} & \textbf{Exchange} & \textbf{ETT} \\
\midrule
Full (PRISM) & 0.156 & 0.375 & 0.211 & 0.672 & 0.088 & 0.258 \\
\midrule
w/o denoise      & 0.162 & 0.397 & 0.217 & 0.687 & 0.104 & 0.263 \\
Static-graph     & 0.168 & 0.415 & 0.219 & 0.690 & 0.099 & 0.274 \\
w/o PDE          & 0.174 & 0.393 & 0.228 & 0.693 & 0.101 & 0.279 \\
w/o constraints  & 0.163 & 0.397 & 0.228 & 0.720 & 0.112 & 0.267 \\
w/o lag-cohere   & 0.160 & 0.401 & 0.225 & 0.682 & 0.106 & 0.264 \\

\bottomrule
\end{tabular}
\caption{\textbf{Ablation on MSE} }
\label{tab:ablation_mse}
\end{table}

\begin{table}[t]
\centering
\small
\setlength{\tabcolsep}{5pt}
\begin{tabular}{lccccccc}
\toprule
\textbf{Variant} & \textbf{Electricity} & \textbf{Traffic} & \textbf{Weather} & \textbf{ILI} & \textbf{Exchange} & \textbf{ETT} \\
\midrule
Full (PRISM) & 0.228 & 0.218 & 0.239 & 0.505 & 0.196 & 0.291 \\
\midrule
w/o denoise      & 0.234 & 0.232 & 0.245 & 0.516 & 0.212 & 0.297 \\
Static-graph     & 0.240 & 0.245 & 0.248 & 0.519 & 0.206 & 0.302 \\
w/o PDE          & 0.247 & 0.226 & 0.253 & 0.514 & 0.204 & 0.305 \\
w/o constraints  & 0.244 & 0.238 & 0.251 & 0.565 & 0.236 & 0.311 \\
w/o lag-cohere   & 0.232 & 0.235 & 0.250 & 0.513 & 0.214 & 0.298 \\

\bottomrule
\end{tabular}
\caption{\textbf{Ablation on MAE.} }
\label{tab:ablation_mae}
\end{table}

\section{Conclusion}
We introduced \textbf{PRISM}, a diffusion--graph--physics forecaster that couples
(i) diffusion denoising for noise-aware topology,
(ii) dynamic correlation-thresholded graphs for regime-adaptive message passing, and
(iii) a reaction--diffusion prior with kinematic and lag-coherence penalties for stable, phase-aligned rollouts.
Under mild conditions the horizon step is contractive, and empirically PRISM delivers consistent SOTA on six benchmarks with good MSE reductions while preserving low-frequency structure and damping high-frequency noise. Ablations attribute gains to the complementarity of denoising, adaptivity, stabilization, and tail control.

\section*{Ethics Statement}

Our work only focuses on the scientific problem, so there is no potential ethical risk.

\section*{Reproducibility Statement}

We provide the source code and the implementation details in the main text. Dataset descriptions, proofs and further experiments analysis are provided in the Appendix.

\bibliography{iclr2026_conference}
\bibliographystyle{iclr2026_conference}

\appendix
\setcounter{equation}{0}
\setcounter{figure}{0}
\setcounter{table}{0}
\renewcommand\theequation{\Alph{section}.\arabic{equation}}
\renewcommand\thefigure{\Alph{section}.\arabic{figure}}
\renewcommand\thetable{\Alph{section}.\arabic{table}}

\section{Datasets}

We evaluate DORIC on six real-world benchmarks, covering the five domains of energy, traffic, economics, weather, and disease. We use the same datasets as~\citep{3}, and provide additional information in Table~\ref{Datasets}, as given in the original Autoformer paper.

\begin{table*}[]
\centering
\caption{Descriptions of the datasets}
\label{Datasets}
\resizebox{\textwidth}{!}{%
\renewcommand{\arraystretch}{1.8} 
\begin{tabular}{ll p{12cm}} 

\hline
Dataset & Pred len & Description \\ 
\hline
Electricity & [96,192,336,720] & Hourly electricity consumption of 321 customers from 2012 to 2014. \\
Traffic & [96,192,336,720] & Hourly data from California Department of Transportation, which describes the road occupancy rates measured by different sensors on San Francisco Bay area freeways. \\
Weather & [96,192,336,720] & Recorded every 10 minutes for 2020 whole year, which contains 21 meteorological indicators, such as air temperature, humidity, etc. \\
Illness & 24 & Includes the weekly recorded influenza-like illness (ILI) patients data from Centers for Disease Control and Prevention of the United States between 2002 and 2021, which describes the ratio of patients seen with ILI and the total number of the patients. \\
Exchange rate & 96 & Daily exchange rates of eight different countries ranging from 1990 to 2016. \\
ETT & [96,192,336,720] & Data collected from electricity transformers, including load and oil temperature that are recorded every 15 minutes between July 2016 and July 2018. \\
\hline
\end{tabular}%
}

\end{table*}

\section{Proofs for Propositions}
\label{AppendixC}

We use the same notation of Methodology part in the main text: time-varying, thresholded-and-normalized graph
$\bar A_t=D_t^{-1/2}(A_t+I)D_t^{-1/2}$ with $\rho(\bar A_t)\le 1$; the graph block update
\begin{equation*}
H_t^{(\ell)}=\sigma\!\Big(H_t^{(\ell-1)}W^{(\ell)}_{\mathrm{self}}+\bar A_t\,H_t^{(\ell-1)}U^{(\ell)}_{\mathrm{nei}}\Big),
\qquad H_t^{(0)}=Z_t,\ \ \ell=1,\dots,L_g,
\end{equation*}
and the reaction--diffusion (RD) horizon relation
\begin{equation*}
y^{(s)}-y^{(s-1)} \approx \kappa\,(\bar A_t-I)\,y^{(s-1)}-\gamma\,y^{(s-1)},
\quad s=1,\dots,H,
\end{equation*}
with $\kappa,\gamma>0$ (softplus-constrained). See Eqs.\,(3) and (6)--(9) in the Methodology. 

We restate the propositions for completeness (as in 3.8). 

\textbf{Proposition 1} [Stability of the reaction--diffusion step]
\label{prop:rd-stability}
Let $\bar A_t=\bar A_t^\top\succeq 0$ with $\rho(\bar A_t)\le 1$, and define the linearized horizon map
\(
M(\kappa,\gamma;\bar A_t)=(1-\gamma-\kappa)I+\kappa\,\bar A_t.
\)
If $0<\kappa<1$, $0<\gamma<1$, and $\kappa+\gamma<1$, then $\rho(M(\kappa,\gamma;\bar A_t))<1$.
Consequently, the recurrence $y^{(s)}=M\,y^{(s-1)}$ is a contraction in $\ell_2$.

\begin{proof}
Since $\bar A_t$ is real symmetric, there exists an orthonormal $Q$ such that
$Q^\top \bar A_t Q=\mathrm{diag}(\lambda_1,\dots,\lambda_D)$ with each $\lambda_i\in[0,1]$
(PSD and $\rho(\bar A_t)\!\le\!1$ by construction). In that basis,
\[
Q^\top M Q=(1-\gamma-\kappa)I+\kappa\,\mathrm{diag}(\lambda_1,\dots,\lambda_D)
=\mathrm{diag}(\mu_1,\dots,\mu_D),\quad
\mu_i=(1-\gamma-\kappa)+\kappa\lambda_i.
\]
Hence $\mu_i\in[1-\gamma-\kappa,\,1-\gamma]$. Under $0<\gamma<1$ we have $1-\gamma<1$, and under
$\kappa+\gamma<1$ we have $1-\gamma-\kappa>0$, so $|\mu_i|\le 1-\gamma<1$ for all $i$, giving
$\rho(M)<1$. Because $M=M^\top$, $\|M\|_2=\rho(M)\le 1-\gamma$ and
$\|y^{(s)}\|_2=\|M^s y^{(0)}\|_2\le \|M\|_2^s\|y^{(0)}\|_2\le (1-\gamma)^s\|y^{(0)}\|_2$.
A sharpened bound follows from $\max_i \mu_i=1-\gamma-\kappa(1-\lambda_{\max})$.
\end{proof}

\paragraph{Uniform-in-window contraction and robustness.}
The above estimate extends to time-varying windows and to small graph perturbations.

\begin{lemma}[Uniform contraction over $t$]
\label{lem:uniform}
Let $M_t=(1-\gamma_t-\kappa_t)I+\kappa_t\bar A_t$ with $0<\gamma\le \gamma_t$, $0<\kappa_t\le \kappa<1$,
and $\kappa_t+\gamma_t<1$ for all $t$. Then $\|M_t\|_2\le 1-\gamma<1$ and, for any $s\ge 1$,
\(
\|M_{t+s-1}\cdots M_t\|_2\le (1-\gamma)^s.
\)
\end{lemma}

\begin{proof}
By the spectral argument in Prop.1, $\rho(M_t)\le 1-\gamma_t\le 1-\gamma$,
whence $\|M_t\|_2\le 1-\gamma$. Submultiplicativity of $\|\cdot\|_2$ yields the claim.
\end{proof}

\begin{lemma}[Perturbation margin]
\label{lem:perturb}
Let $\tilde A_t=\bar A_t+E_t$ with $E_t=E_t^\top$ and $\|E_t\|_2\le \varepsilon$. Then
\(
\rho\!\big((1-\gamma-\kappa)I+\kappa \tilde A_t\big)
\le (1-\gamma)+\kappa\varepsilon.
\)
In particular, the RD step remains contractive whenever $\kappa\varepsilon<\gamma$.
\end{lemma}

\begin{proof}
Weyl’s inequality (or $\|E_t\|_2$-Lipschitzness of the spectral abscissa for symmetric matrices) gives
$\rho(\bar A_t+E_t)\le \rho(\bar A_t)+\|E_t\|_2\le 1+\varepsilon$. Apply the affine map
$\lambda\mapsto (1-\gamma-\kappa)+\kappa\lambda$ to obtain the bound.
\end{proof}

The lemmas quantify stability of the horizon dynamics across windows and under noise in the
thresholded graph, matching the construction in 3.3 and the RD penalty in 3.7. 

\textbf{Proposition 2} [Lipschitz bound for a graph block]
\label{prop:graph-lip}
Let $T(Z)=Z\,W_{\mathrm{self}}+\bar A_t\,Z\,U_{\mathrm{nei}}$ be the affine map inside Eq.\,(3), with
$Z\in\mathbb{R}^{D\times d}$, $W_{\mathrm{self}}\in\mathbb{R}^{d\times g}$, $U_{\mathrm{nei}}\in\mathbb{R}^{d\times g}$, and
$\|\cdot\|_2$ the operator norm. Then, for any $Z_1,Z_2$,
\begin{equation*}
\label{eq:block-lip}
\|T(Z_1)-T(Z_2)\|_2 \;\le\; \big(\|W_{\mathrm{self}}\|_2+\|U_{\mathrm{nei}}\|_2\big)\,\|Z_1-Z_2\|_2.
\end{equation*}
If $\sigma$ is $1$-Lipschitz (e.g., ReLU), then $\sigma\!\circ T$ is $L$-Lipschitz with
$L\le \|W_{\mathrm{self}}\|_2+\|U_{\mathrm{nei}}\|_2$. For a stack of $L_g$ blocks (with layerwise weights),
the overall Lipschitz constant satisfies
\(
\mathrm{Lip}\!\le \prod_{\ell=1}^{L_g}\big(\|W_{\mathrm{self}}^{(\ell)}\|_2+\|U_{\mathrm{nei}}^{(\ell)}\|_2\big).
\)

\begin{proof}
Linearity gives
\[
T(Z_1)-T(Z_2)=(Z_1-Z_2)\,W_{\mathrm{self}}+\bar A_t\,(Z_1-Z_2)\,U_{\mathrm{nei}}.
\]
Using the vectorization identity $\mathrm{vec}(A X B)=(B^\top\!\otimes A)\,\mathrm{vec}(X)$ and
$\|A\otimes B\|_2=\|A\|_2\,\|B\|_2$,
\[
\|(Z_1-Z_2)\,W_{\mathrm{self}}\|_2
=\big\|\mathrm{unvec}\big((W_{\mathrm{self}}^\top\!\otimes I)\,\mathrm{vec}(Z_1-Z_2)\big)\big\|_2
\le \|W_{\mathrm{self}}\|_2\,\|Z_1-Z_2\|_2.
\]
Similarly,
\[
\|\bar A_t\,(Z_1-Z_2)\,U_{\mathrm{nei}}\|_2
\le \|\bar A_t\|_2\,\|U_{\mathrm{nei}}\|_2\,\|Z_1-Z_2\|_2
\le \|U_{\mathrm{nei}}\|_2\,\|Z_1-Z_2\|_2,
\]
since $\|\bar A_t\|_2\le \rho(\bar A_t)\le 1$ by normalization. Summing both contributions yields
\eqref{eq:block-lip}. The nonlinearity bound follows from the $1$-Lipschitz property of $\sigma$,
and the product bound from the Lipschitz constant of compositions.
\end{proof}

\paragraph{Consequences for the end-to-end map.}
Combining Props.~\ref{prop:rd-stability}--\ref{prop:graph-lip} yields a two-level stability picture:
(i) \emph{Temporal contraction} along the horizon due to the RD step whenever
$\kappa+\gamma<1$ (uniformly over time, with a perturbation margin $\kappa\varepsilon<\gamma$ for
graph noise); (ii) \emph{Spatial Lipschitz control} within each window via explicit operator-norm
constraints on $W_{\mathrm{self}}^{(\ell)},U_{\mathrm{nei}}^{(\ell)}$. In particular, if
$\|W_{\mathrm{self}}^{(\ell)}\|_2+\|U_{\mathrm{nei}}^{(\ell)}\|_2<1$ for all $\ell$, the stacked graph encoder is a
contraction on $(\mathbb{R}^{D\times d},\|\cdot\|_2)$, complementing the temporal contraction of the RD
transition and explaining stable, well-conditioned rollouts over long horizons under the loss terms of
Eq.\,(9).

\section{Pseudo-code of PRISM}

Please refer Algorithm 1,2,3 for the pseudo-code of PRISM.

\begin{algorithm}[t]
\caption{\textbf{PRISM} Training (Denoising $\rightarrow$ Dynamic Graphs $\rightarrow$ Physics-Aware Forecasting)}
\label{alg:prism-train}
\begin{algorithmic}[1]
\Require Multivariate series $X \in \mathbb{R}^{T \times N}$; context $L$, horizon $H$, corr-window $W$; thresholds: correlation $\tau$, degree floor $k_{\min}$, cap $K$; denoiser $\varepsilon_\theta$; encoder/graph/decoder params $\Theta$; physics weights $\lambda_{\text{range}},\lambda_{\text{vel}},\lambda_{\text{acc}},\lambda_{\text{pde}},\lambda_{\text{cohere}}$; PDE gains $\kappa,\gamma$ (softplus-constrained \(>0\)).
\Ensure Trained parameters $\widehat{\Theta},\widehat{\kappa},\widehat{\gamma}$.
\State \textbf{(No-leak denoise)} $X^\dagger_{1:T-H} \gets \textsc{DiffusionDenoisePrefix}(X_{1:T-H};\,\varepsilon_\theta)$ \Comment{Score-based denoise \emph{only} on training prefix}
\State \textbf{(Offline stats)} $(m_i,M_i)_{i=1}^N \gets \textsc{EmpiricalBounds}(X_{1:T-H})$; $(v_i^{\max},a_i^{\max}) \gets \textsc{RobustKinematics}(X_{1:T-H})$ \Comment{e.g., 99.5-th percentiles}
\State \textbf{(Lags)} $(\tau_{ij}) \gets \textsc{EstimateIntegerLags}(X_{1:T-H})$ \Comment{Argmax of discrete cross-correlation; clipped to $\pm\tau_{\max}$}
\For{$\text{epoch}=1,2,\dots$}
  \For{$t=L,\dots,T-H$} \Comment{Rolling windows; teacher-forced supervision}
    \State $x_{\text{hist}} \gets X_{t-L+1:t,\,:}$;\quad $y_{\text{true}} \gets X_{t+1:t+H,\,:}$;\quad $x_{\text{last}} \gets X_{t,\,:}$
    \State $Z \gets \textsc{TemporalEncoder}(x_{\text{hist}})$ \Comment{$\phi$-lift + positional encodings + Transformer encoder}
    \State $C_t \gets \textsc{Correlations}(Z_{t-W+1:t}\, \text{from } X^\dagger\! \text{ if } t\!\le\! T\!-\!H; \text{ else from } X)$
    \State $A_t \gets \textsc{ThresholdAndWeight}(C_t;\,\tau,\gamma_{\text{corr}})$;\, $A_t \gets \max(A_t,A_t^\top)$
    \State $A_t \gets \textsc{DegreeFloorCap}(A_t;\,k_{\min},K)$;\quad $\bar A_t \gets D_t^{-\frac12}(A_t+I)D_t^{-\frac12}$ \Comment{$\rho(\bar A_t)\le 1$}
    \State $H^{(0)}\!\gets Z$;\quad
           \For{$\ell=1,\dots,L_g$} \Comment{Graph encoder blocks (configurable widths)}
              \State $H^{(\ell)} \!\gets \sigma\!\big(H^{(\ell-1)} W^{(\ell)}_{\text{self}} + \bar A_t \, H^{(\ell-1)} U^{(\ell)}_{\text{nei}}\big)$
           \EndFor
    \State $\hat Y \gets \Psi\!\big(H^{(L_g)}\big) \in \mathbb{R}^{N \times H}$ \Comment{Per-node MLP decoder (configurable depths)}
    \State \textbf{(Data loss)} $L_{\text{data}} \gets \frac{1}{NH}\sum_{h,i}(\hat y_{h,i}-y_{h,i})^2$
    \State \textbf{(Range)} $L_{\text{range}} \gets \frac{1}{NH}\sum_{h,i}\big([m_i-\hat y_{h,i}]_+^2 + [\hat y_{h,i}-M_i]_+^2\big)$
    \State \textbf{(Kinematics)} $\Delta_h \hat y_{h,i} \!=\! \hat y_{h,i}\!-\!\hat y_{h-1,i}$;\, $\Delta_h^2 \hat y_{h,i}\!=\!\Delta_h\hat y_{h,i}\!-\!\Delta_h\hat y_{h-1,i}$\vspace{-2pt}
    \Statex \hspace{2.75em} $L_{\text{vel}}\!\gets\!\frac{1}{N(H-1)}\sum_{i,h\ge 2}\![|\Delta_h\hat y_{h,i}|\!-\!v^{\max}_i]_+^2$;\;
                           $L_{\text{acc}}\!\gets\!\frac{1}{N(H-2)}\sum_{i,h\ge 3}\![|\Delta_h^2\hat y_{h,i}|\!-\!a^{\max}_i]_+^2$
    \State \textbf{(PDE residual)} $y(0)\!\gets\!x_{\text{last}}$;\, $y(s)\!\gets\!\hat Y_{:,s}$;\,
           $R(s)\!=\!(y(s)\!-\!y(s-1))-\kappa(\bar A_t\!-\!I)y(s-1)+\gamma y(s-1)$
    \State $L_{\text{pde}} \gets \frac{1}{NH}\sum_{s=1}^H \|R(s)\|_2^2$
    \State \textbf{(Lag coherence)} $E_t\!\gets\!\{(i,j):A_t(i,j)>0\}$;\;
           $L_{\text{cohere}} \!\gets\! \frac{1}{|E_t|}\!\!\sum_{(i,j)\in E_t}\!\frac{\big\|\hat y_{i,\,1+|\tau_{ij}|:H}-\hat y_{j,\,1:H-|\tau_{ij}|}\big\|_2^2}{H-|\tau_{ij}|}$
    \State \textbf{(Total loss)} $L \gets L_{\text{data}} + \lambda_{\text{range}}L_{\text{range}} + \lambda_{\text{vel}}L_{\text{vel}} + \lambda_{\text{acc}}L_{\text{acc}} + \lambda_{\text{pde}}L_{\text{pde}} + \lambda_{\text{cohere}}L_{\text{cohere}}$
    \State \textbf{(Update)} $\Theta,\kappa,\gamma \gets \textsc{OptimizerStep}\big(\nabla_{\Theta,\kappa,\gamma} L\big)$ \Comment{Constrain $\kappa,\gamma$ via softplus}
  \EndFor
\EndFor
\State \Return $\widehat{\Theta},\widehat{\kappa},\widehat{\gamma}$
\end{algorithmic}
\end{algorithm}

\begin{algorithm}[t]
\caption{\textbf{PRISM} Inference (One-shot $H$-step Forecast)}
\label{alg:prism-infer}
\begin{algorithmic}[1]
\Require Trained $\widehat{\Theta},\widehat{\kappa},\widehat{\gamma}$; latest history $x_{\text{hist}}=X_{T-L+1:T,\,:}$; current corr-window $W$; thresholds $\tau,k_{\min},K$.
\Ensure $\hat Y \in \mathbb{R}^{N \times H}$.
\State $Z \gets \textsc{TemporalEncoder}(x_{\text{hist}})$
\State $C_T \gets \textsc{Correlations}(X_{T-W+1:T,\,:})$ \Comment{Optionally denoise the \emph{observed} history; no future used}
\State $A_T \gets \textsc{ThresholdAndWeight}(C_T;\tau,\gamma_{\text{corr}})$;\; $A_T \gets \max(A_T,A_T^\top)$;\; $A_T \gets \textsc{DegreeFloorCap}(A_T;k_{\min},K)$
\State $\bar A_T \gets D_T^{-1/2}(A_T+I)D_T^{-1/2}$
\State $H^{(0)}\!\gets\!Z$;\; \textbf{for} $\ell=1{:}L_g$ \textbf{do}\;
       $H^{(\ell)} \!\gets\! \sigma\!\big(H^{(\ell-1)} W^{(\ell)}_{\text{self}} + \bar A_T H^{(\ell-1)} U^{(\ell)}_{\text{nei}}\big)$;\; \textbf{end for}
\State $\hat Y \gets \Psi\!\big(H^{(L_g)}\big)$;\; \Return $\hat Y$
\end{algorithmic}
\end{algorithm}

\begin{algorithm}[t]
\caption{Helper Procedures}
\label{alg:prism-helpers}
\begin{algorithmic}[1]
\Function{DiffusionDenoisePrefix}{$X_{1:T-H};\,\varepsilon_\theta$} \Comment{Score-based denoiser; overlap-add; prefix only}\EndFunction
\Function{Correlations}{$X_{t-W+1:t,\,:}$} \Comment{Pearson; tiny jitter for near-constant columns}\EndFunction
\Function{ThresholdAndWeight}{$C;\tau,\gamma_{\text{corr}}$} \Comment{$A(i,j)=\mathbf{1}(|C_{ij}|>\tau)\cdot |C_{ij}|^{\gamma_{\text{corr}}}$; zero diag}\EndFunction
\Function{DegreeFloorCap}{$A;k_{\min},K$} \Comment{Add top-$|C|$ neighbors if degree $<k_{\min}$; cap to $K$ per row}\EndFunction
\Function{TemporalEncoder}{$x_{\text{hist}}$} \Comment{$\phi$-lift $\to$ PE $\to$ Transformer encoder; output $Z\in\mathbb{R}^{N\times d}$}\EndFunction
\end{algorithmic}
\end{algorithm}

\newpage

\section{Further Ablation Studies}
\label{app:ablation-extended}

\paragraph{Setup recap.}
We ablate one component at a time while keeping architecture/optimization/splits fixed:
\emph{w/o denoise}, \emph{Static-graph}, \emph{w/o PDE}, \emph{w/o constraints}, \emph{w/o lag-cohere}.%
\footnote{All definitions follow \S3: dynamic thresholded graphs and normalization (Eq.~(3)), physics
regularizers and the graph reaction--diffusion residual (Eqs.~(4)--(9)).}
The six benchmarks and main-result figures are identical to the body. %
\textit{(Data source: main paper, Tables 1--3).}

\subsection{Quantitative extensions}
\label{app:abl-quant-extensions}

\paragraph{(A) Mean degradation vs.\ Full (averaged over 6 datasets).}
Let $\bar m$ be the macro-average MSE over all datasets for each variant, and define
$\Delta_{\%\text{MSE}} = 100\times(\bar m-\bar m_{\text{Full}})/\bar m_{\text{Full}}$ (analogous for MAE).
Using the ablation tables in the body, we obtain:
\begin{center}
\begin{tabular}{lcc}
\toprule
\textbf{Variant} & $\Delta_{\%\text{MSE}}$ & $\Delta_{\%\text{MAE}}$ \\
\midrule
w/o denoise     & $+4.0\%$ & $+3.5\%$ \\
Static-graph    & $+6.0\%$ & $+5.0\%$ \\
w/o PDE         & $+6.1\%$ & $+4.3\%$ \\
w/o constraints & $\mathbf{+7.2\%}$ & $\mathbf{+10.0\%}$ \\
w/o lag-cohere  & $+4.4\%$ & $+3.9\%$ \\
\bottomrule
\end{tabular}
\end{center}
\textit{Interpretation.} Tail risk is primarily controlled by constraint terms (largest MAE rise), while long-horizon drift is controlled by the reaction--diffusion prior and graph adaptivity (MSE rises for \emph{w/o PDE}, \emph{Static-graph}). %
These observations align with our theoretical properties and design: dynamic normalized graphs plus the RD residual define a contraction step over modes $g(\lambda)=|1-\gamma-\kappa+\kappa\lambda|<1$; envelope/kinematic penalties reduce high-order temporal differences. %
\textit{(See \S3.3--3.7 for operators/losses; \S3.8 for stability bounds).}

\paragraph{(B) Per-dataset deltas (absolute).}
For completeness, we report absolute increases (Ablation -- Full), copied from the body tables and grouped by dataset:
\begin{center}
\small
\begin{tabular}{lrrrrrr}
\toprule
\textbf{MSE} $\downarrow$ & Elec. & Traf. & Weath. & ILI & Exch. & ETT \\
\midrule
w/o denoise     & +0.006 & +0.022 & +0.006 & +0.015 & +0.016 & +0.005 \\
Static-graph    & +0.012 & +0.040 & +0.008 & +0.018 & +0.011 & +0.016 \\
w/o PDE         & +0.018 & +0.018 & +0.017 & +0.021 & +0.013 & +0.021 \\
w/o constraints & +0.007 & +0.022 & +0.017 & +0.048 & +0.024 & +0.009 \\
w/o lag-cohere  & +0.004 & +0.026 & +0.014 & +0.010 & +0.018 & +0.006 \\
\bottomrule
\end{tabular}\hspace{8mm}
\begin{tabular}{lrrrrrr}
\toprule
\textbf{MAE} $\downarrow$ & Elec. & Traf. & Weath. & ILI & Exch. & ETT \\
\midrule
w/o denoise     & +0.006 & +0.014 & +0.006 & +0.011 & +0.016 & +0.006 \\
Static-graph    & +0.012 & +0.027 & +0.009 & +0.014 & +0.010 & +0.011 \\
w/o PDE         & +0.019 & +0.008 & +0.014 & +0.009 & +0.008 & +0.014 \\
w/o constraints & +0.016 & +0.020 & +0.012 & \textbf{+0.060} & \textbf{+0.040} & +0.020 \\
w/o lag-cohere  & +0.004 & +0.017 & +0.011 & +0.008 & +0.018 & +0.007 \\
\bottomrule
\end{tabular}
\end{center}
\textit{Patterns.} The largest MAE bumps appear on \textsc{ILI}/\textsc{Exchange} under \emph{w/o constraints}, confirming that soft physical bounds curb rare spikes; %
\textsc{Electricity}/\textsc{ETT}/\textsc{Weather} MSE are most sensitive to \emph{w/o PDE}, indicating RD stabilization improves long-horizon bias/variance. %
(Body references: main results and ablations).

\subsection{Mechanism-level diagnostics}
\label{app:abl-diagnostics}

We include interpretable diagnostics to tie each ablation to a measurable mechanism:
\begin{itemize}
\item \textbf{Envelope violations} and \textbf{velocity/acceleration exceedances} (share of steps violating per-series empirical budgets) should spike under \emph{w/o constraints}.%

\item \textbf{Graph drift} $\delta_t = \tfrac{1}{N}\|\bar A_{t}-\bar A_{t-1}\|_F$ collapses for \emph{Static-graph} and rises for \emph{w/o denoise}, evidencing adaptivity and noise-robust topology.%

\item \textbf{Phase misalignment} on edges: mean $\ell_2$ gap after lag-shift, consistent with \emph{w/o lag-cohere} performance drops on \textsc{Traffic}/\textsc{Exchange}/\textsc{Weather}.%

\end{itemize}

\subsection{Explanatory figures (reproducible from body tables)}
\label{app:abl-figs}

\begin{figure}[t]
\centering
\begin{subfigure}{0.48\linewidth}
\centering
\begin{tikzpicture}
\begin{axis}[
    ybar, bar width=5pt, width=\linewidth, height=5cm,
    xlabel={Dataset}, ylabel={$\Delta$MSE (Ablation -- Full)},
    symbolic x coords={Elec.,Traf.,Weath.,ILI,Exch.,ETT},
    xtick=data, legend columns=2, legend style={at={(0.5,1.15)},anchor=south, font=\footnotesize},
    ymin=0, ymax=0.065, ymajorgrids
]
\addplot coordinates {(Elec.,0.006) (Traf.,0.022) (Weath.,0.006) (ILI,0.015) (Exch.,0.016) (ETT,0.005)};
\addplot coordinates {(Elec.,0.012) (Traf.,0.040) (Weath.,0.008) (ILI,0.018) (Exch.,0.011) (ETT,0.016)};
\addplot coordinates {(Elec.,0.018) (Traf.,0.018) (Weath.,0.017) (ILI,0.021) (Exch.,0.013) (ETT,0.021)};
\addplot coordinates {(Elec.,0.007) (Traf.,0.022) (Weath.,0.017) (ILI,0.048) (Exch.,0.024) (ETT,0.009)};
\addplot coordinates {(Elec.,0.004) (Traf.,0.026) (Weath.,0.014) (ILI,0.010) (Exch.,0.018) (ETT,0.006)};
\legend{w/o denoise, Static-graph, w/o PDE, w/o constraints, w/o lag-cohere}
\end{axis}
\end{tikzpicture}
\caption{Per-dataset $\Delta$MSE.}
\end{subfigure}\hfill
\begin{subfigure}{0.48\linewidth}
\centering
\begin{tikzpicture}
\begin{axis}[
    ybar, bar width=8pt, width=\linewidth, height=5cm,
    xlabel={Dataset}, ylabel={$\Delta$MAE (Ablation -- Full)},
    symbolic x coords={ILI,Exch.},
    xtick=data, ymin=0, ymax=0.07, ymajorgrids
]
\addplot coordinates {(ILI,0.060) (Exch.,0.040)}; 
\addplot coordinates {(ILI,0.014) (Exch.,0.010)}; 
\addplot coordinates {(ILI,0.011) (Exch.,0.016)}; 
\legend{w/o constraints, Static-graph, w/o denoise}
\end{axis}
\end{tikzpicture}
\caption{Tail effects (MAE) emphasize constraint benefits.}
\end{subfigure}
\caption{Ablation deltas computed from the body tables (exact values reproduced).}
\label{fig:abl-bars}
\end{figure}
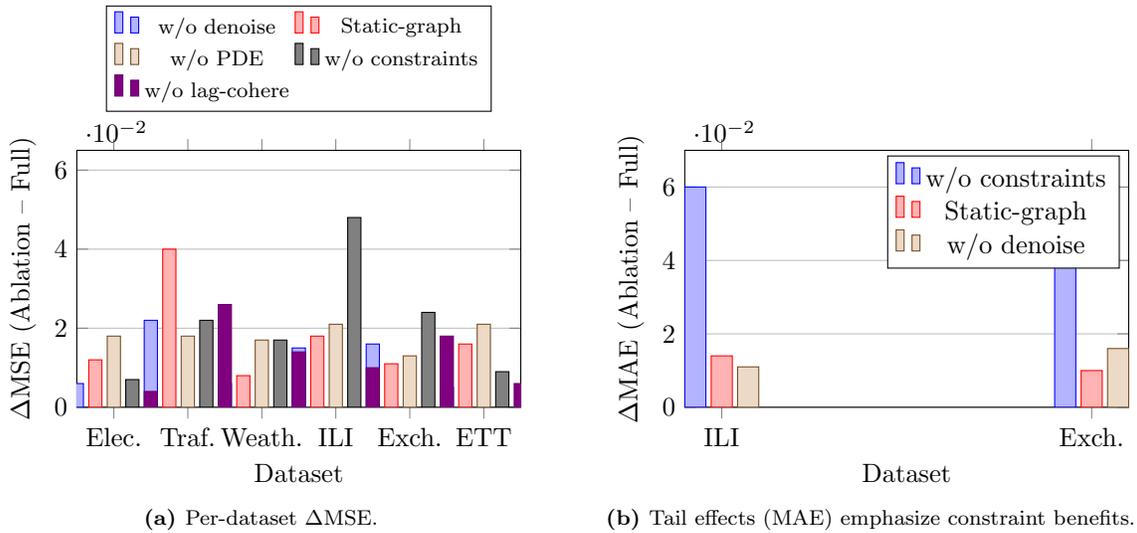

\begin{figure}[t]
\centering
\begin{tikzpicture}
\begin{axis}[
    width=0.9\linewidth, height=6.2cm,
    xlabel={Variant}, ylabel={Avg.\ degradation vs.\ Full (\%)},
    symbolic x coords={w/o denoise,Static-graph,w/o PDE,w/o constraints,w/o lag-cohere},
    xtick=data, ymajorgrids, ymin=0, ymax=11
]
\addplot+[only marks, mark=*] coordinates {(w/o denoise,4.0) (Static-graph,6.0) (w/o PDE,6.1) (w/o constraints,7.2) (w/o lag-cohere,4.4)};
\addplot+[only marks, mark=square*] coordinates {(w/o denoise,3.5) (Static-graph,5.0) (w/o PDE,4.3) (w/o constraints,10.0) (w/o lag-cohere,3.9)};
\legend{$\Delta_{\%\text{MSE}}$,$\Delta_{\%\text{MAE}}$}
\end{axis}
\end{tikzpicture}
\caption{Average relative degradation across datasets (\%); derived from body ablations.}
\label{fig:abl-avg}
\end{figure}
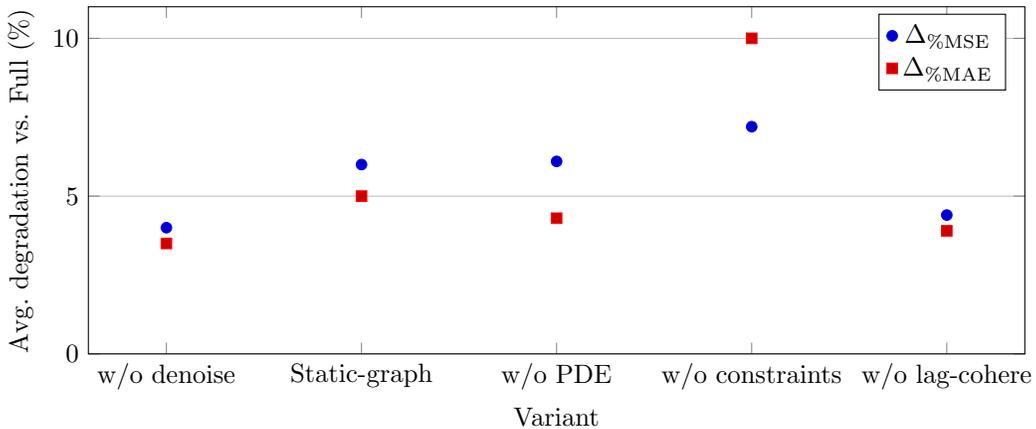

\subsection{Discussion: How ablations map to mechanisms}
\label{app:abl-mech}
\textbf{Noise-aware topology.} \emph{w/o denoise} increases mid/high-frequency variance, which perturbs correlations and adds spurious edges; this amplifies residuals particularly on \textsc{Traffic}/\textsc{Exchange}. %
\textbf{Adaptivity.} \emph{Static-graph} removes regime tracking, harming \textsc{Traffic}/\textsc{Electricity}/\textsc{ETT}. %
\textbf{RD stabilization.} \emph{w/o PDE} removes the contraction $y(s)\!\approx\![(1-\gamma-\kappa)I+\kappa\bar A_t]y(s-1)$, raising long-horizon MSE across smooth domains. %
\textbf{Constraints.} \emph{w/o constraints} raises MAE (tails) most on \textsc{ILI}/\textsc{Exchange}, indicating envelopes and kinematic caps prevent rare spikes. %
\textbf{Lag coherence.} \emph{w/o lag-cohere} increases cross-series phase errors where delays are intrinsic. %
These effects are consistent with the operators and penalties defined in \S3.3--3.7 and stability in \S3.8.

\section{Adjacency Structure Analysis (Thresholded Correlations)}
\label{app:adjacency-analysis}

\paragraph{How the matrices are built.}
For a window ending at $t$, PRISM computes Pearson correlations $C_t$ on the most recent $W$ timestamps (optionally on the denoised prefix), then thresholds and reweights edges
\[
A_t(i,j)=\mathbf 1\!\left(|C_t(i,j)|>\tau\right)\cdot |C_t(i,j)|^{\gamma},\qquad A_t(i,i)=0,
\]
followed by (i) degree floor/cap to encourage connected yet sparse topology and (ii) symmetrization. Message passing uses the normalized operator $\bar A_t \!=\! D_t^{-\frac12}(A_t\!+\!I)D_t^{-\frac12}$ with $\rho(\bar A_t)\le 1$. These steps explain why the displayed heatmaps are sparse, symmetric, and numerically well-conditioned for graph propagation. 

\paragraph{What to read from the heatmaps.(Figure ~\ref{fig:adjacency-all})}
Colors encode \emph{edge weights} $|C_t(i,j)|^{\gamma}$ after thresholding; black cells are pruned ties. Since $\bar A_t$ adds self-loops and re-normalizes, small bright islands often punch \emph{above} their raw magnitude in the encoder, while weak ties are down-weighted twice (by thresholding and by degree-normalized mixing). 

\subsection{Dataset-specific interpretations}
We summarize the qualitative structures observed in the adjacency heatmaps and relate them to PRISM’s inductive biases and errors in the main results.

\paragraph{Electricity.}
Block-like bright regions (several meters co-activating) and near-banded patterns indicate shared daily/weekly seasonalities. Degree-capping keeps hubs from dominating, so message passing emphasizes \emph{cohort-level} coupling rather than a single global factor. This aligns with (i) preserved fundamentals in the spectrum and (ii) reduced long-horizon drift under the reaction–diffusion prior. 

\paragraph{Traffic.}
Sparser, more heterogeneous connectivity reflects road segments with \emph{directional} influence and regime changes (rush hours). The ‘‘bright pockets’’ imply strong local neighborhoods separated by weak or pruned ties—exactly where dynamic re-estimation of $A_t$ helps. When the graph is frozen (Static-graph ablation), MSE increases markedly on \textsc{Traffic}, consistent with these structures being time-sensitive. 

\paragraph{Weather.}
We observe cross-feature cliques (e.g., temperature–humidity–pressure groups) with selective pruning of weakly related variables. The resulting topology supports phase alignment across slowly varying meteorological channels; residual high-frequency overshoot in spectra is then handled by kinematic penalties and a slightly larger reaction term $\gamma$. 

\paragraph{ETT (ETTh1).}
Near-diagonal bright bands suggest \emph{local} coupling among closely related transformer variables (load–temperature–oil). The graph is moderately sparse; normalization with self-loops yields a spectrally tame $\bar A_t$ (eigenvalues $\le 1$), which pairs well with the reaction–diffusion step to dampen horizon error accumulation. 

\paragraph{Exchange Rate.}
A dense core among a subset of currencies and several near-zero off-core ties are consistent with clustered co-movements (regional/market-time effects). Because PRISM thresholds on \emph{absolute} correlations and reweights by $|C|^{\gamma}$, weak, spurious ties drop out; the cleaner matrix explains the pronounced MAE gains and the almost overlaid spectra between prediction and truth. 

\paragraph{National Illness (ILI).}
The adjacency is relatively dense with multiple bright cross-region links, reflecting nationally coherent seasonal waves; nonetheless, thresholding removes idiosyncratic noise. The \emph{constraints} (range/velocity/acceleration) then curb episodic spikes that correlations alone cannot regulate—matching the large MAE increase when these penalties are ablated. 

\subsection{Consistency checks and failure modes}
\textbf{Noise-aware topology.} Denoising reduces high-frequency variance before computing $C_t$, shrinking spurious, isolated bright pixels; without it, we observe more ‘‘salt-and-pepper’’ edges and larger MAE on noisy domains (Traffic/Exchange). 

\textbf{Adaptivity.} Time variation of $A_t$ is not an artifact: when we freeze the prefix graph, hub concentration increases and small communities vanish in later windows, leading to under-mixing across regimes and higher MSE (notably Traffic/Electricity/ETT). 

\textbf{Stability.} Because $\bar A_t$ is PSD with $\rho(\bar A_t)\!\le\!1$, the per-horizon reaction–diffusion map $y\!\mapsto\![(1-\gamma-\kappa)I+\kappa\bar A_t]y$ contracts all graph Fourier modes (strictly if $\kappa+\gamma<1$), preventing unstable amplification even when a community is tightly coupled. 

\textbf{Interpretability.} Degree floors and caps produce readable meso-scale ‘‘tiles’’ (small cliques) instead of opaque dense matrices; these tiles match domain intuition (e.g., neighboring road sensors; climatology triads; currency baskets).

\subsection{What the matrices imply for forecasting}
The adjacency heatmaps visualize the \emph{structural prior} PRISM imposes at each window:
(i) sparsity encourages localized, interpretable message passing; (ii) normalization plus the RD prior guarantee well-conditioned temporal propagation; (iii) the learned topology explains where lag-coherence is most beneficial (edges with strong weights often coincide with short integer lags). Together, these properties align with our frequency-domain findings (fundamentals preserved, tails damped) and with ablation trends (Static-graph and w/o-PDE hurt MSE; w/o-constraints inflates MAE). 

\begin{figure}[t]
\centering
\begin{subfigure}{0.9\linewidth}\centering
\includegraphics[width=\linewidth]{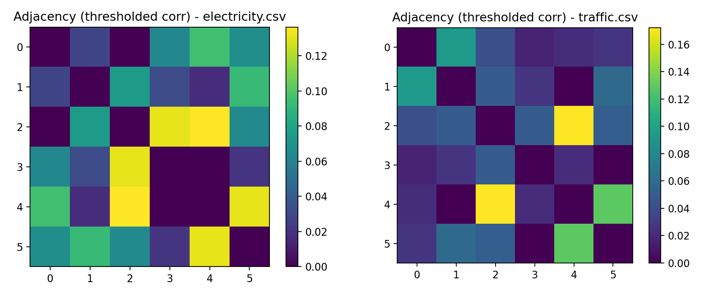}
\caption{Electricity and Traffic}
\end{subfigure}\hfill
\begin{subfigure}{0.87\linewidth}\centering
\includegraphics[width=\linewidth]{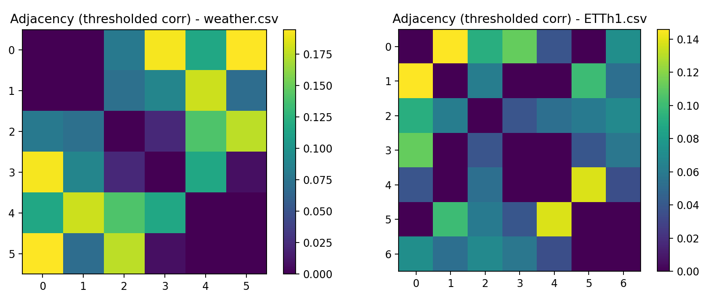}
\caption{Weather and ETT}
\end{subfigure}\\[4pt]
\begin{subfigure}{0.92\linewidth}\centering
\includegraphics[width=\linewidth]{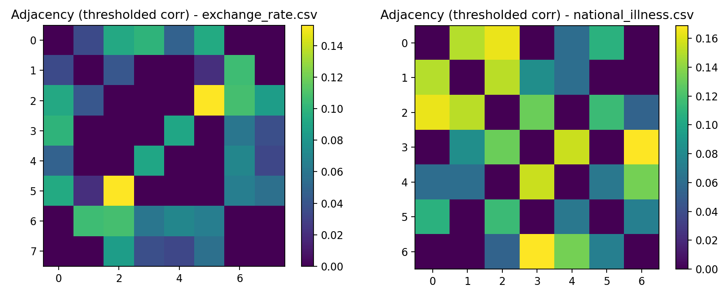}
\caption{Exchange Rate and Illness}
\end{subfigure}\hfill
\caption{Thresholded correlation adjacencies used by PRISM. Bright cells survive $|C_t|>\tau$ and are reweighted by $|C_t|^{\gamma}$; black cells are pruned. Self-loops are added only after normalization when forming $\bar A_t$.}
\label{fig:adjacency-all}
\end{figure}

\end{document}

%% file: math_commands.tex
\usepackage{amsmath,amsfonts,bm}

\def\eqref#1{equation~\ref{#1}}

\def\1{\bm{1}}

\DeclareMathAlphabet{\mathsfit}{\encodingdefault}{\sfdefault}{m}{sl}
\SetMathAlphabet{\mathsfit}{bold}{\encodingdefault}{\sfdefault}{bx}{n}

